\documentclass[11pt]{article}

\usepackage[imports,opt,theoremstyles]{commands}	
\usepackage{xcolor}

\newcommand{\opnorm}[1]{\norm{#1}_{\operatorname{S}_\infty}}
\newcommand{\nucnorm}[1]{\norm{#1}_{\operatorname{S}_1}}
\newcommand{\frobnorm}[1]{\norm{#1}_{\operatorname{F}}}
\newcommand{\dnorm}[1]{\norm{#1}_{\Ut, *}}

\newcommand{\nameeqref}[1]{\hyperref[#1]{(\nameref{#1})}}
\newcommand{\ourAlg}{\textsc{AngularMuown}}

\newcommand{\muown}{\textsc{Muown}}
\newcommand{\Mc}{\mathcal{M}}
\newcommand{\Tc}{\mathcal{T}}
\newcommand{\Tan}[1]{\Tc_{#1}\Mc}
\newcommand{\Abf}{\mathbf{A}}
\newcommand{\Dbf}{\mathbf{D}}
\newcommand{\Ebf}{\mathbf{E}}
\newcommand{\Gbf}{\mathbf{G}}
\newcommand{\Mbf}{\mathbf{M}}
\newcommand{\Obf}{\mathbf{O}}

\newcommand{\Rbf}{\mathbf{R}}
\newcommand{\Sbf}{\mathbf{S}}
\newcommand{\Ubf}{\mathbf{U}}
\newcommand{\Vbf}{\mathbf{V}}
\newcommand{\Wbf}{\mathbf{W}}
\newcommand{\Zbf}{\mathbf{Z}}

\genCommandsName{E}{Vector}{\Ebf}
\genCommandsName{G}{Vector}{\Gbf}
\genCommandsName{M}{Vector}{\Mbf}
\genCommandsName{O}{Vector}{\Obf}
\genCommandsName{R}{Vector}{\Rbf}
\genCommandsName{S}{Vector}{\Sbf}
\genCommandsName{U}{Vector}{\Ubf}
\genCommandsName{W}{Vector}{\Wbf}
\genCommandsName{Z}{Vector}{\Zbf}

\DeclareMathOperator{\DiagOp}{Diag}
\DeclareMathOperator{\diagOp}{diag}
\DeclareMathOperator{\ProjOp}{Proj}
\DeclareMathOperator{\RetrOp}{Retr}
\DeclareMathOperator{\GradOp}{grad}

\DeclareMathOperator{\AdamOp}{Adam}
\newcommand{\Diag}[1]{\DiagOp\!\pare{#1}}
\newcommand{\diag}[1]{\diagOp\!\pare{#1}}
\newcommand{\Proj}[2]{\ProjOp_{#1}\!\pare{#2}}
\newcommand{\Retr}[2]{\RetrOp_{#1}\!\pare{#2}}
\newcommand{\rownorm}[1]{\norm{#1}_{\mathrm{row}}}
\newcommand{\obman}{\mathcal{M}_{\mathrm{ob}}}
\newcommand{\tang}[1]{\mathrm{T}_{#1}\obman}
\newcommand{\onev}{\mathbf{1}}
\newcommand{\zerob}{\mathbf{0}}
\newcommand{\sphere}[1]{\mathbb{S}^{#1}}
\providecommand{\inner}[2]{\left\langle #1, #2 \right\rangle}

\definecolor{hlgreen}{RGB}{173,209,53}   

\newcommand{\hlg}[1]{\colorbox{hlgreen}{#1}}

\newcommand{\equalcontrib}{\textsuperscript{\textbf{*}}}
\newcommand{\equalcontribfootnote}{%
  \begingroup
  \renewcommand{\thefootnote}{\fnsymbol{footnote}}%
  \footnotetext[1]{Equal contribution.}%
  \endgroup
}					
\usepackage[margin=2.5cm]{geometry}					
\usepackage{wrapfig}								
\usepackage{booktabs}

\usepackage{tabularx}                               
\usepackage{makecell}
\usepackage{multirow}
\usepackage{graphicx}								
\usepackage{subcaption}								
\usepackage{algorithm}								
\usepackage{algpseudocode}
\usepackage{tikz}                                   
\usetikzlibrary{arrows.meta,calc}

\usepackage[sorting=ynt,
			natbib=true,
			style=authoryear-comp,
			maxcitenames=2,
            maxbibnames=99,
			uniquelist=false,
			uniquename=false,
			backend=biber]{biblatex}
\title{\vspace{-5mm}Muown Implicitly Performs Angular Step-size Decay}
\author{%
\begin{tabular}{@{}c@{\hspace{3em}}c@{}}
    \begin{tabular}[t]{c}
    \textbf{Florian Hübler}\equalcontrib\\
    Department of Computer Science\\
    ETH Zurich, Switzerland\\
    \texttt{florian.huebler@inf.ethz.ch}
    \end{tabular}
    &
    \begin{tabular}[t]{c}
    \textbf{Kai Lion}\equalcontrib\\
    Department of Computer Science\\
    ETH Zurich, Switzerland\\
    \texttt{kai.lion@inf.ethz.ch}
    \end{tabular}
    \\[20mm]
    \begin{tabular}[t]{c}
    \textbf{Antonio Orvieto}\\
    ELLIS Institute Tübingen, MPI-IS\\
    Tübingen AI Center, Germany\\
    \texttt{antonio@tue.ellis.eu}
    \end{tabular}
    &
    \begin{tabular}[t]{c}
    \textbf{Niao He}\\
    Department of Computer Science\\
    ETH Zurich, Switzerland\\
    \texttt{niao.he@inf.ethz.ch}
    \end{tabular}
\end{tabular}%
}
\date{}

\begin{document}

\maketitle
\equalcontribfootnote

\begin{abstract}
    Matrix-aware optimizers such as \textsc{Muon} and \textsc{Muown} have recently shown strong empirical performance for pre-training Transformers. In particular, \textsc{Muown} separates each weight matrix into row magnitudes and an un-normalized direction variable, updating the former with \textsc{Adam} and the latter with \textsc{Muon}. We show that the directional update of \textsc{Muown} is equivalent to a Riemannian step on the normalized directions, while the magnitude of the un-normalized parameterization only modulates the angular step size. This explains the step-size stability of \textsc{Muown} and suggests making the angular step size explicit. The resulting method, \ourAlg, optimizes directly over the normalized directions and uses a schedulable angular multiplier decoupled from the radial magnitude update. \ourAlg\ improves over \textsc{Muown} and, at the time of writing, a preliminary version is leading the per-optimizer category of the modded nanoGPT speedrunning competition. Further experiments on Qwen2-0.5B, and 1.1B parameter mixture-of-experts models confirm the algorithm scales beyond small models. An implementation of the algorithm is available at \url{https://github.com/fhueb/angular-muown}.
\end{abstract}

\section{Introduction}
\label{sec:intro}
\begin{wrapfigure}{r}{0.5\textwidth}
	\vspace{-\baselineskip}
	\centering
	\includegraphics[width=\linewidth]{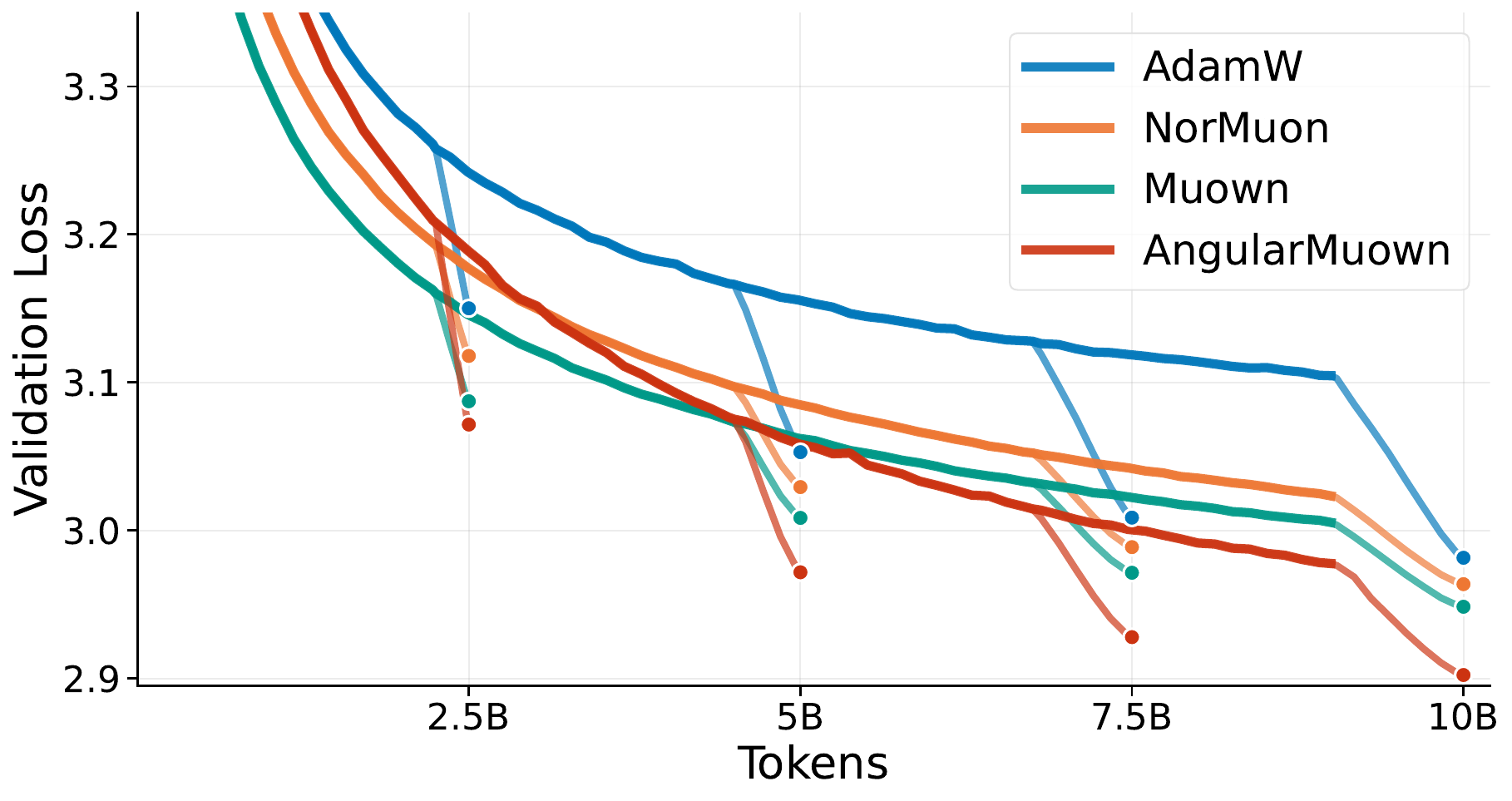}
	\vspace*{-7mm}
	\caption{Validation loss of a 124M Transformer checkpoint-branching pre-trained on FineWeb-Edu. \ourAlg\ substantially improves over tuned \textsc{Muown}, NorMuon, and \textsc{AdamW} across the token budget.}
	\label{fig:teaser}
	\vspace{-1\baselineskip}
\end{wrapfigure}

For over a decade, \textsc{Adam} \citep{kingma_adam_2015} has been the default optimizer for deep learning. Recently, a family of \emph{matrix-aware} optimizers that exploit the two-dimensional structure of weight matrices has gained traction, particularly for Transformers whose parameters are dominated by dense matrices. \textsc{Muon}, one of the simplest and most effective variants, performs normalized steepest-descent updates in the spectral norm \citep{jordan2024muon}. This realizes the per-step size prescribed by the spectral condition for stable feature learning \citep{yang_spectral_2024}, but it only bounds the change $\opnorm{\Delta\Wbf}$ and leaves the growth of $\opnorm{\Wbf}$ along the trajectory uncontrolled.

\textsc{Muown} \citep{lion_muown_2026} addresses this shortcoming by re-parameterizing each weight into its row-magnitudes and directions, i.e., $\Wbf = \Diag{g/\rownorm{\Rbf}}\Rbf$, where $\rownorm{\Abf} \in \R^m$ denotes the vector of Euclidean row norms of $\Abf \in \R^{m \times n}$. This corresponds to an optimizer-internal weight normalization \citep{salimans_weight_2016}, which separates a per-row magnitude $g$ from a direction $\Rbf$. The magnitudes are trained with \textsc{Adam} and the direction with \textsc{Muon}. Beyond improving performance, this simple change makes training remarkably robust to large stepsizes that would otherwise cause \textsc{Muon} to diverge.

We argue that this robustness is not incidental but geometric. The directional update on $\Rbf$, conducted in the ambient space $\R^{m\times n}$, can equivalently be written as a Riemannian update on the row-oblique manifold $\mathcal M = \set{\Abf \in \R^{m \times n} \suchthat \norm{\Abf}_{\mathrm{row}} = \mathbf{1}}$. In this view, the directional update step becomes a per-row \emph{angular} displacement whose size scales inversely with the row norm. As the row norms grow during training, this angular step is silently annealed, which explains \textsc{Muown}'s stepsize stability. It is also exactly the role weight decay plays implicitly elsewhere, modulating the angular step size by controlling the weight norm \citep{laarhoven_l2_2017, kosson_rotational_2024}.

\begin{wrapfigure}{r}{0.51\textwidth}
	\vspace{-\baselineskip}
	\begin{minipage}{0.51\textwidth}
	\begin{algorithm}[H]
		\caption{Simplified \ourAlg}
		\label{alg:simplified_ourAlg}
		{\linespread{1.3}\selectfont
		\small
		\begin{algorithmic}
			\Require Gradient $\Gbf = \nabla_\Wbf \mathcal L(\Wbf)$, momentum buffer $\Mbf$, learning rate $\eta$, momentum $\beta_1, \beta_2$, shape scale $s_{m, n}$, \hlg{angular step multiplier $\kappa_t$} \vspace{2mm}
			\State $\Wbf = \Diag{g} \Ubf$ \Comment{Reparameterize internally}
            \State $\operatorname{grad}_g \gets \diag{\Gbf \Ubf^\top}$
			\State $\operatorname{grad}_\Ubf \gets \Diag{g}\pare{ \Gbf - \Diag{\operatorname{grad}_g}\Ubf }$\vspace{2mm}
			\State $\Mbf \gets \beta_1 \Mbf + \operatorname{grad}_\Ubf$ \Comment{Update $\Ubf$ with \textsc{Muon}}
			\State $\Obf \gets \mathrm{Orth}\pare{\operatorname{grad}_\Ubf + \beta_1 \Mbf }$ 
			\State $\Ubf \gets \hlg{$\operatorname{RowNormalize}$}\pare{\Ubf - \eta \hlg{$\kappa_t$} s_{m, n} \Obf}$ \vspace{2mm} 
			\State $g \gets \AdamOp(g, \eta, \beta_1, \beta_2)$ \Comment{Update $g$ with \textsc{Adam}} \vspace{2mm}
			\State $\Wbf \gets \Diag{g}\,\Ubf$ \Comment{Re-compose}
		\end{algorithmic}
		}
	\end{algorithm}
	\end{minipage}
	\vspace{-1.5\baselineskip}
\end{wrapfigure}
Making this implicit structure explicit yields \ourAlg. Instead of updating the direction $\Rbf$ in the ambient space, we explicitly keep it normalized and perform Riemannian \textsc{Muon} updates on $\Ubf = \Diag{\rownorm{\Rbf}}^{-1} \Rbf$. In particular, the directional stepsize now governs the angle of the directional update and can be scheduled explicitly.
The method thus preserves the spectral update that makes \textsc{Muon} effective while exposing the angular step size as an explicit design choice. A preliminary version of \ourAlg\ performed strongly in the per-optimizer category of the modded nanoGPT speedrun, suggesting that the Riemannian interpretation is not only explanatory but also practically useful.\footnote{As of June 16th 2026, this preliminary version is the leading per-optimizer result, see \textsc{Muown-RowNormControl} in Figure 2: \url{https://github.com/KellerJordan/modded-nanogpt/tree/23f758f123df4b8dd5b5fe64c4c9070f6ef33b52/records/track_3_optimization\#notable-results-history}.}

\paragraph{Prior public disclosure.} The preliminary version of \ourAlg\ was publicly disclosed in our May 8, 2026 modded-nanoGPT speedrun \href{https://github.com/KellerJordan/modded-nanogpt/pull/288}{submission}. That submission used \muown’s internal decomposition into a per-row gain $g$ and stored direction $v$, and introduced a scheduled stored-direction norm, explicitly described there as \emph{angular step-size attenuation}: increasing $\lVert v\rVert$ while recomposing through $v/\lVert v\rVert$ reduces the effective angular displacement on the unit sphere. The present work formalizes this mechanism as implicit Riemannian angular step-size decay and replaces the stored-norm gauge schedule by an explicit angular multiplier $\kappa_t$.

\paragraph{Contributions.} 
In this work we identify a key implicit property of \textsc{Muown} which potentially explains its success, make it explicit and provide strong empirical support for the resulting algorithm \ourAlg.

\begin{itemize}
    \item We show that \textsc{Muown} implicitly performs Riemannian updates on the direction, and that its stepsize stability arises from an implicit annealing of the angular step size by row-norm growth of the direction. Building on this view, we introduce \ourAlg, which turns the angular step size into an explicit, schedulable quantity that is independent of the radial update (\Cref{sec:motivation_and_prelims,sec:riemannian_muown}).
    \item We conduct extensive experiments on transformer-based architectures that show a consistent improvement of \ourAlg\ compared to \textsc{AdamW}, (\text{Nor}-)\textsc{Muon}, and \textsc{Muown}. In particular, we observe a significant speed-up over all baselines across dense 124M and 500M, as well as 1.1B mixture-of-experts models (Section~\ref{sec:experiments}).
    \item We provide theoretical support for the directional update. In particular, in the non-convex setting, we show that an $\eps$-stationary point is reached in at most $\Oc\pare{\min\set{m, n} \Dz L \sigma^2 \eps^{-4}}$ gradient evaluations under reasonable assumptions (Section~\ref{sec:angular_muown.conv_guarantee}).
\end{itemize}

\section{Motivation}
\label{sec:motivation_and_prelims}
\definecolor{muongreen}{RGB}{0,150,105}
\definecolor{muonpink}{RGB}{200,90,150}
\definecolor{muonblue}{RGB}{0,105,180}
\definecolor{softgray}{RGB}{90,90,90}

\begin{wrapfigure}{r}{0.3\textwidth}
    \centering
    \vspace{-\baselineskip}
    \begin{tikzpicture}[
        x=1cm,y=1cm,
        >=Latex,
        line cap=round,
        line join=round,
        bluevec/.style={draw=blue!70!black, very thick, -{Latex[length=2.2mm]}},
        grevec/.style={draw=green!45!black, very thick, -{Latex[length=2.2mm]}},
        magray/.style={draw=magenta!65!black, dashed, thick},
        magarc/.style={draw=magenta!65!black, thick},
        labb/.style={text=blue!70!black, font=\small},
        labg/.style={text=green!45!black, font=\small},
        labm/.style={text=magenta!65!black, font=\small}
    ]
    
        \def\phi{50}
        \def\rs{1.5}
        \def\rl{3.5}
        \def\a{1}
    
        \coordinate (O)   at (0,0);
        \coordinate (Rs)  at ({\rs*cos(\phi)},{\rs*sin(\phi)});
        \coordinate (Rl)  at ({\rl*cos(\phi)},{\rl*sin(\phi)});
        \coordinate (Rsp) at ({\rs*cos(\phi)-\a*sin(\phi)},{\rs*sin(\phi)+1.2*\a*cos(\phi)});
        \coordinate (Rlp) at ({\rl*cos(\phi)-\a*sin(\phi)},{\rl*sin(\phi)+1.2*\a*cos(\phi)});
    
        \pgfmathsetmacro{\phis}{atan2(\rs*sin(\phi)+\a*cos(\phi),\rs*cos(\phi)-\a*sin(\phi))}
        \pgfmathsetmacro{\phil}{atan2(\rl*sin(\phi)+\a*cos(\phi),\rl*cos(\phi)-\a*sin(\phi))}
    
        \draw[black, very thick]
            ({\rs*cos(18)},{\rs*sin(18)}) arc[start angle=18,end angle=116,radius=\rs];
        \draw[black, very thick]
            ({\rl*cos(16)},{\rl*sin(16)}) arc[start angle=16,end angle=104,radius=\rl];
    
        \draw[bluevec] (O) -- (Rs);
        \draw[bluevec] (O) -- (Rl);
    
        \draw[grevec] (Rs) -- (Rsp);
        \draw[grevec] (Rl) -- (Rlp);
    
        \draw[magray] (O) -- (Rsp);
        \draw[magray] (O) -- (Rlp);
    
        \def\anglesmallrad{1.05}
        \def\anglelargerad{2.75}
        
        \coordinate (ThetaSmallEnd) at (\phis:\anglesmallrad);
        \coordinate (ThetaLargeEnd) at (\phil:\anglelargerad);
        
        \draw[magarc](\phi:\anglesmallrad)
            arc[start angle=\phi,end angle=\phis,radius=\anglesmallrad];
        
        \draw[magarc](\phi:\anglelargerad)
            arc[start angle=\phi,end angle=1.01*\phil,radius=\anglelargerad];
        
        \node[labm, anchor=east, xshift=-2pt] at ($(ThetaSmallEnd)+(0.15,0.05)$) {$\theta_{\mathrm{small}}$};
        \node[labm, anchor=east, xshift=-2pt] at ($(ThetaLargeEnd)+(0.1,0.1)$) {$\theta_{\mathrm{large}}$};
    
        \fill[blue!70!black] (Rs)  circle (1.5pt);
        \fill[blue!70!black] (Rl)  circle (1.5pt);
        \fill[green!45!black] (Rsp) circle (1.5pt);
        \fill[green!45!black] (Rlp) circle (1.5pt);
    
        \node[labb, anchor=west] at ($(Rs)+(0.05,0.02)$) {$R_{\mathrm{small}}$};
        \node[labb, anchor=west] at ($(Rl)+(0.03,0.02)$) {$R_{\mathrm{large}}$};
        \node[labg, anchor=west] at ($(Rlp)+(0.28,-0.2)$) {$\Delta R=\eta \Obf$};
    
    \end{tikzpicture}
    \caption{Impact of weight norm on angular update size.}
    \label{fig:angular_due_to_magnitude}
\end{wrapfigure}
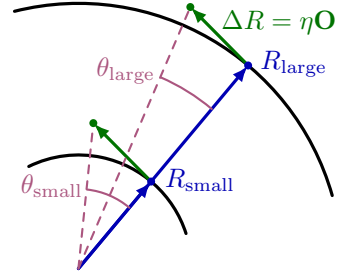

For normalized or scale-invariant weights, the direction alone determines the feature the row (or neuron) encodes. Given this observation, the meaningful effective step size is the per-step \emph{angular} change of that row.  While the row-norm does not impact the feature representation, it does impact optimization dynamics by modulating how strongly a fixed Euclidean update rotates that direction. A large norm makes the same update induce a smaller angular change, and a small norm induces a larger one. Thus, the norm acts as an implicit controller of the angular step size.
This perspective explains part of the effect of weight decay in normalized networks. Rather than acting only as classical regularization, weight decay counteracts norm growth and thereby prevents angular updates from collapsing, yielding an equilibrium of angular updates \citep{laarhoven_l2_2017,NormMatters2018Hoffer,kosson_rotational_2024}. 

\muown\ inherits this mechanism in a more direct form. Its internal parameterization
\begin{equation*}
    \Wbf = \Diag{g/\rownorm{\Rbf}}\Rbf
      = \Diag g \Ubf,
\end{equation*}
with $ \Ubf = \Diag{\rownorm{\Rbf}}^{-1}\Rbf$, represents the row magnitude as $g$ whereas $\rownorm{\Rbf}$ is a gauge variable. Changing it does not change $\Wbf$, but it does change the optimization dynamics. Since \muown\ updates $\Rbf$ in the ambient space, the same additive update to $\Rbf$ induces a smaller rotation when $\rownorm{\Rbf}$ is large, as illustrated in Figure~\ref{fig:angular_due_to_magnitude}. Mathematically, the internal directional update $\Rtp \gets \Rt - \eta \Obf$ can \emph{equivalently} be written as $\widetilde{\Ubf}_{\tp} \gets \Ut - \eta \Diag{\rownorm{\Rt}}^{-1} \Obf$, making this dependence explicit.

This provides a simple explanation for \muown's stepsize stability: growth of the unnormalized direction norm silently anneals the angular update. More importantly, it reveals that \muown's parameterization already exposes the angular step size as the natural quantity to control, opening up a new opportunity: rather than letting the angular schedule emerge implicitly from the growth of the row norm, we can schedule it explicitly. \ourAlg\ realizes this by optimizing directly over the row-normalized directions and replacing \muown's implicit norm-induced annealing with an explicit angular multiplier, making weight decay redundant in the process.

\section{\texorpdfstring{\ourAlg}{AngularMuown}}
\label{sec:riemannian_muown}
In this section, we discuss the motivation and different parts of \ourAlg{} in more detail. We first make the geometry that is implicit in \textsc{Muown} explicit: its directional update is a Riemannian steepest-descent step on the oblique manifold (\Cref{subsec:oblique,subsec:implicit}), which turns the Euclidean \textsc{Muon} step into a per-row \emph{angular} displacement modulated by the row scale (\Cref{subsec:angular}). Promoting this modulation to an explicit schedule yields \ourAlg\ (\Cref{subsec:riemuown}).

\subsection{The Oblique-Manifold Parameterization}
\label{subsec:oblique}

\textsc{Muown} reparameterizes each weight as $\Wbf = \Diag{g / \rownorm{\Rbf}}\,\Rbf$ \citep{salimans_weight_2016}, separating a per-row scale $g \in \R^m$ from a direction matrix $\Rbf \in \R^{m \times n}$. Since only the row-normalized $\Diag{1/\rownorm{\Rbf}}\Rbf$ enters the weight, the directional degrees of freedom are exactly the matrices with unit rows, i.e. the \emph{oblique manifold} $\obman \coloneqq \{ \Ubf \in \R^{m \times n} : \rownorm{\Ubf} = \onev_m \} = (\sphere{n-1})^m$, the product of $m$ unit spheres embedded in $\R^{m\times n}$ \citep{absil_optimization_2008, boumal_introduction_2023}. We therefore write the parameterization directly on the manifold,
\begin{equation}
	\label{eq:reparam}
	\Wbf(g, \Ubf) = \Diag{g}\, \Ubf,
	\qquad \Ubf \in \obman,\quad g \in \R^m ,
\end{equation}
which matches \textsc{Muown} under $\Ubf = \Diag{1/\rownorm{\Rbf}}\Rbf$ and identifies $g$ with the row magnitudes $\rownorm{\Wbf}$. The tangent space at $\Ubf$ collects matrices with rows orthogonal to those of $\Ubf$, $\tang{\Ubf} = \{ \Vbf : \diag{\Vbf \Ubf^\top} = \zerob \}$, where $\diag{\cdot}$ extracts the diagonal. Under the Frobenius metric, the orthogonal projection onto $\tang{\Ubf}$ removes the per-row radial component, and we return to $\obman$ by the row-normalization retraction:
\begin{equation}
	\label{eq:proj-retr}
	\Proj{\Ubf}{\Abf} \coloneqq \Abf - \Diag{\diag{\Abf \Ubf^\top}}\,\Ubf ,
	\qquad
	\Retr{\Ubf}{\Vbf} \coloneqq \Diag{ 1 / \rownorm{\Ubf + \Vbf} }\,(\Ubf + \Vbf) .
\end{equation}
The projection subtracts $\inner{\Abf_i}{\Ubf_i}\Ubf_i$ from each row $\Abf_i$, and the retraction is exactly the row renormalization \textsc{Muown} performs when recomputing the effective weight \citep{absil_optimization_2008}.

\subsection{Muown Implicitly Performs Riemannian Optimization}
\label{subsec:implicit}

Differentiating through \eqref{eq:reparam} splits the gradient $\Gbf = \nabla_\Wbf \mathcal L(\Wbf)$ into a radial and a directional part. The chain rule gives the radial gradient $\nabla_g \mathcal L = \diag{\Gbf \Ubf^\top} \in \R^m$, with $(\nabla_g \mathcal L)_i = \inner{\Gbf_i}{\Ubf_i}$, and the Euclidean directional gradient $\Diag{g}\Gbf$. Projecting the latter onto $\tang{\Ubf}$ yields the \emph{Riemannian gradient} of $\Ubf \mapsto \mathcal L(\Diag{g}\Ubf)$,
\begin{equation}
	\label{eq:riem-grad}
	\operatorname{grad}_\Ubf \mathcal L
	= \Proj{\Ubf}{\Diag{g}\Gbf}
	= \Diag{g}\Proj{\Ubf}{\Gbf}
	= \Diag{g}\pare{ \Gbf - \Diag{\nabla_g \mathcal L}\,\Ubf },
\end{equation}
the second equality holding because $\Diag{g}$ rescales rows and the projection \eqref{eq:proj-retr} acts row-wise. These are exactly the decoupled gradients \textsc{Muown} forms internally: the radial part drives the row scales, the tangent part the directions.

\textsc{Muown} updates the directions by feeding \eqref{eq:riem-grad} through the \textsc{Muon} pipeline (momentum, Nesterov, spectral orthogonalization) and renormalizing. Spectral orthogonalization $\Obf = \argmin_{\opnorm{\Obf} \le 1} \inner{\Vbf}{\Obf}$, computed by Newton--Schulz, returns the normalized spectral steepest-descent direction \citep{ConvexOptimization2004boyd, bernstein_old_2024}. Composing it with the retraction \eqref{eq:proj-retr} shows the \textsc{Muown} direction update is a retracted, spectrally-conditioned Riemannian (quasi-)gradient step on $\obman$,
\begin{equation*}
	\Utp = \Retr{\Ut}{\eta_t\, s_{m,n}\, \Obf_t },
	\qquad
	\Obf_t = \argmin_{\opnorm{\Obf} \leq 1}\!\langle \operatorname{grad}_{\Ut}\mathcal L + \beta \Mt , \Obf \rangle,
\end{equation*}
with the radius $g$ optimized in parallel by \textsc{Adam}, matching the $\ell_\infty$ geometry singled out for the row scales in \citep{lion_muown_2026}, and $s_{m,n}$ the shape scale of \Cref{subsec:riemuown}. Two details depart from a textbook Riemannian method: the momentum $\Mt$ is accumulated in the ambient space rather than parallel-transported, and $\Obf_t$ is retracted directly without re-projecting onto $\tang{\Ut}$, with the retraction absorbing any residual radial component. \textsc{Muown}, though derived from a purely row-magnitude argument, thus already \emph{is} a Riemannian optimizer on $\obman$.

\subsection{From Euclidean Steps to Angular Step Sizes}
\label{subsec:angular}

Section~\ref{sec:motivation_and_prelims} singled out the per-row \emph{angle} as the effective step size of a direction update. The oblique-manifold view makes this angle intrinsic: each update in \eqref{eq:riem-grad} rotates a unit row $\Ubf_i$ along its sphere, so its proper step size is the swept angle, not the Euclidean displacement $\norm{\xi_i}_2$. The following elementary fact converts the Euclidean step the optimizer computes into that angle.

\begin{table}[t]
	\centering
	\caption{Three regimes for the per-row angular step size $\theta_{t,i}$, all sharing the Euclidean tangential step $a_{t,i} \coloneqq \eta_t\, s_{m,n}\, \norm{\Obf_{t,i}^\perp}_2$ with $\Obf_{t,i}^\perp = \Proj{\Ubf_{t,i}}{\Obf_{t,i}}$. Weight decay and \textsc{Muown} both modulate the angle through a norm $\nu_{t,i}$ in the \emph{denominator}, and differ only in that norm's dynamics: decay pins it at a \emph{rotational equilibrium} \citep{kosson_rotational_2024, laarhoven_l2_2017}, whereas \textsc{Muown}'s unregularized stored norm \emph{grows}, annealing the angle as a side effect. \ourAlg\ fixes $\nu_{t,i}\equiv1$ and moves the schedule into the \emph{numerator}, decoupling it from any norm.}
	\label{tab:regimes}
	\begin{tabular}{@{}llll@{}}
		\toprule
		Regime & $\tan\theta_{t,i}$ & Governing norm $\nu_{t,i}$ & Late-training angle \\
		\midrule
		Weight decay (\textsc{Muon}, $\lambda>0$) & $a_{t,i}/\nu_{t,i}$ & $\norm{\Wbf_{t,i}}_2 \to \nu_\infty$ (equilibrium) & plateau, $\theta_{t,i}\to\theta_\infty>0$ \\
		\addlinespace
		Implicit (\textsc{Muown}) & $a_{t,i}/\nu_{t,i}$ & $r_{t,i}\sim\sqrt{t}$ (unregularized growth) & anneals, $\theta_{t,i}\sim t^{-1/2}$ \\
		\addlinespace
		Explicit (\ourAlg) & $\kappa_t\, a_{t,i}$ & $\nu_{t,i}\equiv 1$ (decoupled) & scheduled, $\theta_{t,i}\sim\kappa_t$ \\
		\bottomrule
	\end{tabular}
\end{table}

\begin{proposition}[Angular step size]
	\label{prop:angular-step}
	Let $\Ubf_i \in \sphere{n-1}$ be a unit row and $\xi_i \in \R^n$ a proposed step, split into radial and tangential parts $\xi_i = \inner{\xi_i}{\Ubf_i}\Ubf_i + \xi_i^\perp$ with $\xi_i^\perp = \Proj{\Ubf_i}{\xi_i}$. The angle $\theta_i$ between $\Ubf_i$ and the retracted row $\Retr{\Ubf_i}{\xi_i}$ satisfies
	\begin{equation}
		\label{eq:angular-step}
		\tan \theta_i = \frac{\norm{\xi_i^\perp}_2}{1 + \inner{\xi_i}{\Ubf_i}} .
	\end{equation}
	Only the tangential part turns the row---the radial part merely stretches it and is erased by the renormalization---so for a purely tangential step ($\inner{\xi_i}{\Ubf_i} = 0$) the rotation reduces to $\tan\theta_i = \norm{\xi_i}_2$.
\end{proposition}
\begin{proof}
	The retraction rescales $\Ubf_i + \xi_i$ to unit length without changing its direction, so $\theta_i$ is the angle between $\Ubf_i$ and $\Ubf_i + \xi_i$. Since $\norm{\Ubf_i}_2 = 1$, this vector splits orthogonally as $\Ubf_i + \xi_i = (1 + \inner{\xi_i}{\Ubf_i})\,\Ubf_i + \xi_i^\perp$: its component along $\Ubf_i$ (length $1 + \inner{\xi_i}{\Ubf_i}$) and its perpendicular component (length $\norm{\xi_i^\perp}_2$) are the adjacent and opposite sides of the right triangle subtending $\theta_i$, and their ratio is $\tan\theta_i$.
\end{proof}

The proposition lets us read off the per-step rotation in each parameterization. For the small steps taken in practice the radial denominator $1 + \inner{\xi_i}{\Ubf_i} \approx 1$, so $\tan\theta_i \approx \norm{\xi_i^\perp}_2$ and only the tangential part contributes, mirroring the projection in \eqref{eq:riem-grad}. \ourAlg\ and \textsc{Muown} use the \emph{same} Euclidean update $-\eta_t s_{m,n}\Obf_{t,i}$, but feed a different step $\xi_i$ into the retraction. \ourAlg\ retracts the unit row directly, so $\xi_i = -\eta_t \kappa_t s_{m,n}\Obf_{t,i}$. \textsc{Muown} keeps the \emph{unnormalized} direction and adds the update to it, $\Rbf_{t+1,i} = \Rbf_{t,i} - \eta_t s_{m,n}\Obf_{t,i}$, with unit row $\Ubf_{t,i} = \Rbf_{t,i}/r_{t,i}$ and $r_{t,i} = \norm{\Rbf_{t,i}}_2$. Factoring out the current norm,
\begin{equation*}
	\Rbf_{t+1,i} = r_{t,i}\pare{ \Ubf_{t,i} - \tfrac{\eta_t s_{m,n}}{r_{t,i}}\,\Obf_{t,i} } ,
\end{equation*}
and as normalization discards the positive prefactor $r_{t,i}$, the next unit row $\Ubf_{t+1,i} = \Rbf_{t+1,i}/\norm{\Rbf_{t+1,i}}_2$ is exactly $\Retr{\Ubf_{t,i}}{\xi_i}$ with the \emph{rescaled} step $\xi_i = -(\eta_t s_{m,n}/r_{t,i})\,\Obf_{t,i}$: the larger the stored norm, the smaller the step the unit row actually sees. Substituting each $\xi_i$ into \Cref{prop:angular-step} gives
\begin{equation}
	\label{eq:angular-contrast}
	\underbrace{\tan \theta_{t,i} = \eta_t\, \kappa_t\, s_{m,n}\,\norm{\Obf_{t,i}^\perp}_2}_{\ourAlg}
	\qquad\text{versus}\qquad
	\underbrace{\tan \theta_{t,i} = \eta_t\, s_{m,n}\,\norm{\Obf_{t,i}^\perp}_2 \big/ r_{t,i}}_{\textsc{Muown}} .
\end{equation}
where $\Obf_{t,i}^\perp = \Proj{\Ubf_{t,i}}{\Obf_{t,i}}$. \ourAlg's rotation is fixed by the multiplier $\kappa_t$ and is independent of the row scale, whereas \textsc{Muown}'s carries the extra factor $1/r_{t,i}$ inherited from the \emph{growing} stored norm. As the row norms drift upward under \textsc{Muon}, the mechanism \textsc{Muown} exposes, the diagonal multiplier $1/r_{t,i}$ acts as a per-row angular schedule that decays as a side effect of the row-scale dynamics. \ourAlg{} replaces this implicit, scale-coupled decay with an explicit, independently chosen schedule.

\subsection{The AngularMuown Update}
\label{subsec:riemuown}

\subsubsection{Angular Learning-Rate Multiplier}
\Cref{subsec:angular} suggests an obvious degree of freedom: rather than letting the angular schedule emerge as the side effect $1/r_{t,i}$ of row-norm growth, we prescribe it. We keep $\Ubf \in \obman$ at unit row norm and scale the directional step by an \emph{angular learning-rate multiplier} $\kappa_t \in (0, 1]$, recovering \textsc{Muown}'s implicit decay when $\kappa_t \propto 1/r_{t,i}$ while opening the design space to schedules chosen on their own merit. Motivated by theory, we use an inverse-polynomial,
\begin{equation*}
	\kappa_t^{\mathrm{poly}} = \pare{1 + c\,(t - t_w)_+}^{-p},
\end{equation*}
held at $1$ for $t \le t_w$ warm-up steps. Here $c > 0$ and $p > 0$ are hyperparameters setting the decay rate and shape, but we observe that $c = 0.001, p = 1$ work uniformly. In particular, $\kappa_t^{\mathrm{poly}}$ decays without reference to a fixed endpoint and is therefore compatible with warmup-stable-decay, horizonless training \citep{hu_minicpm_2024}. Crucially $\kappa_t$ acts \emph{only} on the directional step, while the row scales $g$ keep the base rate $\eta_t$ through \textsc{Adam}, so the radial and angular dynamics are scheduled independently. The shape scale $s_{m,n}$ calibrates each directional step on the manifold.

\subsubsection{Shape Scaling for Dimension Independence}
\label{sec:angular_muown.angular_muown.scale_factor}
The remaining factor is the shape scale $s_{m,n}$, which calibrates each directional step on the manifold. We use the spectral-condition choice $s_{m,n} = \sqrt{\max(1, m/n)}$ by default such that the angular update is dimension-independent. By \eqref{eq:angular-contrast}, the per-row angle obeys $\tan\theta_{t,i} = \eta_t\,\kappa_t\,s_{m,n}\,\norm{\Obf_{t,i}^\perp}_2$, so the choice of $s_{m,n}$ sets how a fixed multiplier $\kappa_t$ translates into an actual rotation across layers of different shapes. The dimension dependence enters through $\norm{\Obf_{t,i}^\perp}_2$ and the orthogonalized update $\Obf_t$ has unit singular values, such that $\norm{\Obf_t}_F^2 = \min(m,n)$. Consequently, the energy of a typical row is roughly $\norm{\Obf_{t,i}}_2 \approx \sqrt{\min(m,n)/m} = 1/\sqrt{\max(1, m/n)}$. The spectral-condition choice $s_{m,n} = \sqrt{\max(1, m/n)}$ \citep{yang_spectral_2024} is exactly the reciprocal of this factor, so $s_{m,n}\,\norm{\Obf_{t,i}^\perp}_2 = \Theta(1)$ regardless of $m$ and $n$, rendering the angular update \emph{dimension-independent}: a single schedule $\kappa_t$ then induces the same angular step size on every row regardless of its shape. By contrast, the \textsc{Adam}-RMS-matching choice $s_{m,n} \propto \sqrt{\max(m,n)}$ \citep{liu_muon_2025} leaves a residual $s_{m,n}\,\norm{\Obf_{t,i}^\perp}_2 = \Theta(\sqrt{n})$ that grows with the fan-in, so the same $\kappa_t$ would produce systematically larger rotations on wider rows. In particular this choice of shape scaling gives our angular stepsize $\sst \kappa_t$ a physical meaning: by Proposition~\ref{prop:angular-step} with $\langle \xi_i, \Ubf_i \rangle \approx 0$ the direction is updated by $\tan \theta_i = \sst \kappa_t s_{m,n} \norm{\Obf_{t,i}^\perp}_2 \approx \sst \kappa_t$. Since, for small stepsize, we have $\arctan(x) \approx x + O(x^3)$ this corresponds to a directional update of $\approx {\frac{180}{\pi} \sst \kappa_t}^\circ \approx {57 \sst \kappa_t}^\circ$ degrees.

\Cref{alg:simplified_ourAlg} collects these ingredients. What \ourAlg{} adds over \textsc{Muown} is control. The angular and radial geometries that \textsc{Muown} entangles through the single row norm $r_{t,i}$ become explicit and independently schedulable, which we show in \Cref{sec:experiments} improves perplexity. Moreover, our shape scaling improves hyperparameter transferability over \textsc{Muown}'s by removing implicit shape dependence.

\subsection{Convergence Guarantee}
\label{sec:angular_muown.conv_guarantee}

Finally we provide a convergence guarantee for the directional update of an idealized version of \ourAlg. The idealized version keeps the orthogonalized update in the tangent space, see Algorithm~\ref{alg:idealised_rie_muown}. In the following we denote the Riemannian gradient of $\mathcal L \colon \Mc \to \R$ at $\Ubf \in \Mc$ as $\GradOp \mathcal L \pare{\Ubf}$. The formal assumptions can be found in Appendix~\ref{sec:app.missing_proofs}.

\begin{theorem}[Convergence Guarantee]\label{thm:conv}
    Let Assumptions \ref{assum:lb}, \ref{assum:smooth}, and \ref{assum:noise} hold, and denote $r \coloneqq \min\set{m, n}$. Then Algorithm~\ref{alg:idealised_rie_muown} with parameters
    \begin{equation*}
        \ssi = \sqrt{\frac{\Dz (1-\beta)}{L T}}, \qquad \beta = 1 - \min\set{1, \max\set{T^{-\nicefrac 2 3}, \sqrt{\frac{\Dz L}{r \sigma^2 T}}}}
    \end{equation*}
    satisfies
    \begin{equation*}
        \frac 1 T \iSum \Exp{\frobnorm{\GradOp \mathcal L (\Ut)}}
        \leq 4 \sqrt{\frac{\Dz L}{T}} + 7 \pare{\frac{r \Dz L \sigma^2}{T}}^{\nicefrac 1 4} + 4 \frac{\sqrt r \sigma}{T^{\nicefrac 1 3}}.
    \end{equation*}
\end{theorem}

\begin{figure}[t]
	\centering
	\begin{subfigure}{0.49\textwidth}
		\centering
		\includegraphics[width=\linewidth]{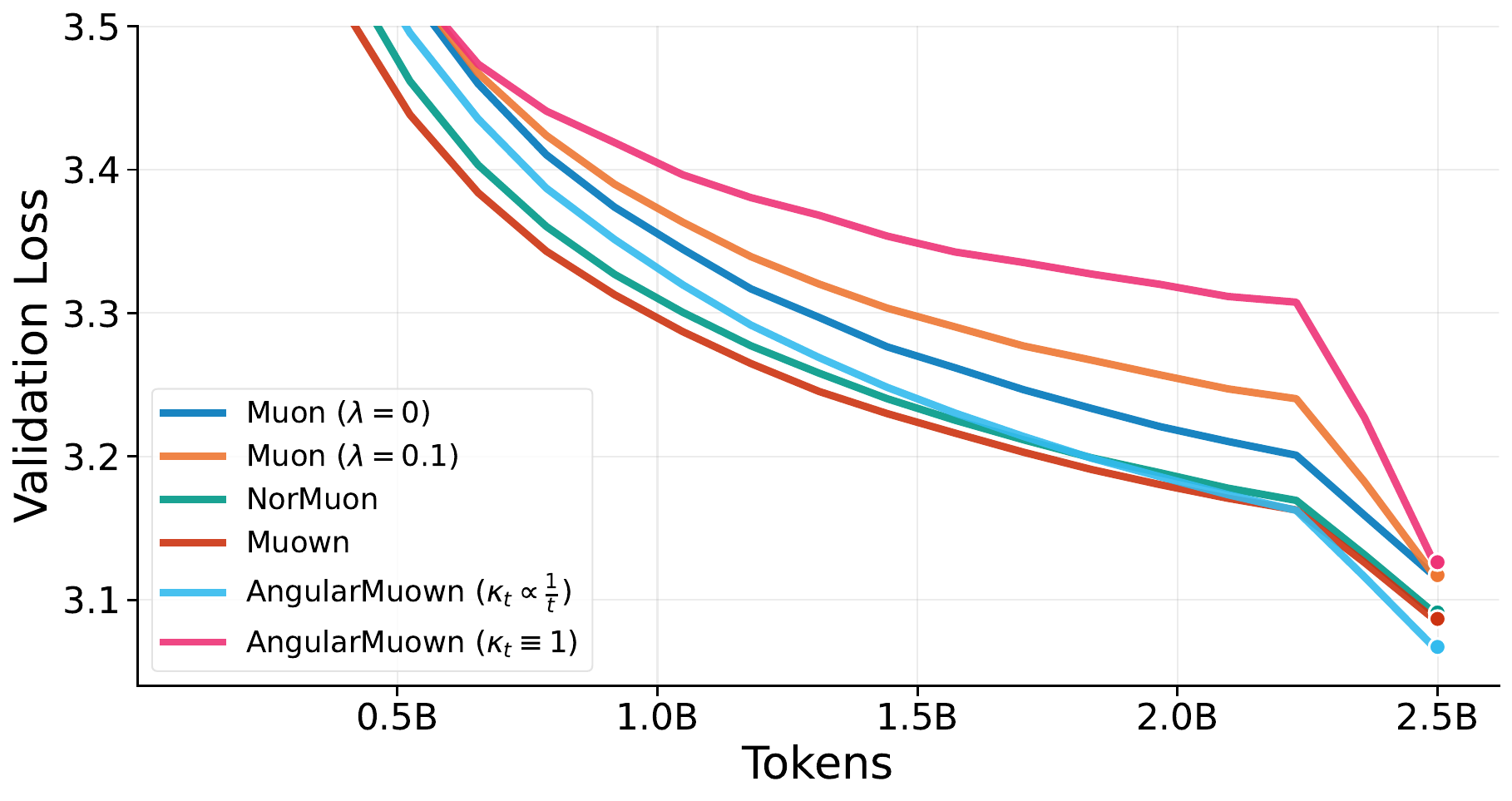}
		\caption{Loss curves.}
		\label{fig:124M-lr}
	\end{subfigure}
	\hfill
	\begin{subfigure}{0.49\textwidth}
		\centering
        \includegraphics[width=\linewidth]{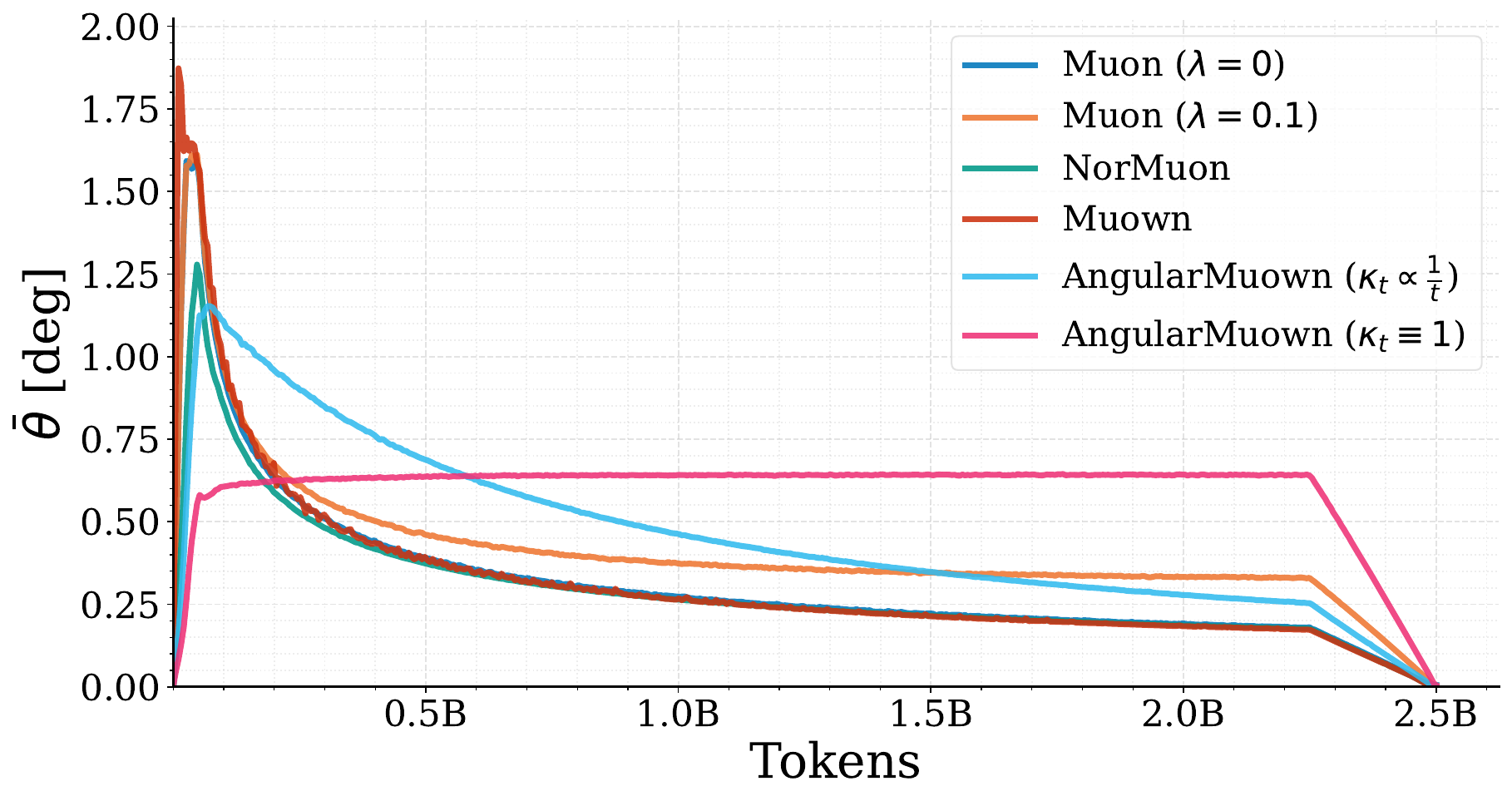}
		\caption{Angular displacement.}
		\label{fig:124M-angular}
	\end{subfigure}
	\caption{Analysis on the 124M model on 2.5B tokens. Weight decay is denoted by $\lambda$, \ourAlg's angular multiplier as $\kappa_t$. The reported runs correspond to the optimally tuned learning rate. \textbf{(a)} Validation loss curves across optimizers. \textbf{(b)} Mean angular displacement $\overline{\theta}$.}
	\label{fig:124M}
\end{figure}

The proof closely follows \muown's proof, but requires some care to handle the retraction and can be found in Appendix~\ref{sec:app.missing_proofs}.
The result implies a leading sample complexity of $\Oc \pare{\min\set{m, n}\Dz L \sigma^2 \eps^{-4}}$.
Surprisingly and contrary to previous observations, we observe that \ourAlg\ actually requires the polynomially decaying stepsize suggested by theory.
The proof also works for general retractions and exponential maps, however we stick to our specific choices for ease of exposition.

\section{Experiments}
\label{sec:experiments}
In this section, we evaluate \ourAlg\ on language-model pre-training with modern architectures on FineWeb-Edu. We conduct experiments to validate the empirical performance of \ourAlg\ across different transformer architectures and sizes. Furthermore we conduct an angular displacement comparison between algorithms to highlight its efficiency, and examine the aglorithms learning rate transfer.

\paragraph{Setup.} All experiments are conducted on nodes with 4 GH200 GPUs with a total of approximately {10'000} GPU hours, including preliminary experiments. Unless specified otherwise, we use PolarExpress \citep{PolarExpress2025amsel} for orthogonalization, the WSD schedule \citep{hu_minicpm_2024} with 100 warmup steps, a cooldown ratio of $0.1$, sequence length $1024$, and batch size $512$.

\begin{figure}[t]
	\centering

	\begin{subfigure}[t]{0.47\textwidth}
		\vspace{0pt}
		\centering
		\begin{minipage}[t][0.22\textheight][c]{\linewidth}
			\centering
			\setlength{\tabcolsep}{3.5pt}
			\renewcommand{\arraystretch}{0.95}
			\begin{tabular}{@{}lccc@{}}
            	\toprule
            	\multirow{2}{*}{$\eta$}
            	& \multicolumn{2}{c}{\ourAlg{}}
            	& \multirow{2}{*}{\textsc{NorMuon}} \\
            	\cmidrule(lr){2-3}
            	& $\kappa_t \propto \nicefrac 1 \ii$ & $\kappa_t \equiv 1$ & \\
            	\midrule
            	\texttt{4e-3} & $2.862$          & $2.700$ & $2.713$ \\
            	\texttt{6e-3} & $2.792$          & $2.694$ & $2.701$ \\
            	\texttt{8e-3} & $2.749$          & $2.698$ & $2.695$ \\
            	\texttt{1e-2} & $2.721$          & $2.708$ & $2.692$ \\
            	\texttt{2e-2} & $2.662$          & $2.767$ & $2.718$ \\
            	\texttt{4e-2} & $\mathbf{2.648}$ & $2.873$ & $2.749$ \\
            	\bottomrule
            \end{tabular}
		\end{minipage}
		\caption{Qwen2-0.5B learning-rate sweep.}
		\label{tab:qwen_results}
	\end{subfigure}
	\hfill
	\begin{subfigure}[t]{0.50\textwidth}
		\vspace{0pt}
		\centering
		\begin{minipage}[t][0.22\textheight][c]{\linewidth}
			\centering
			\includegraphics[width=\linewidth]{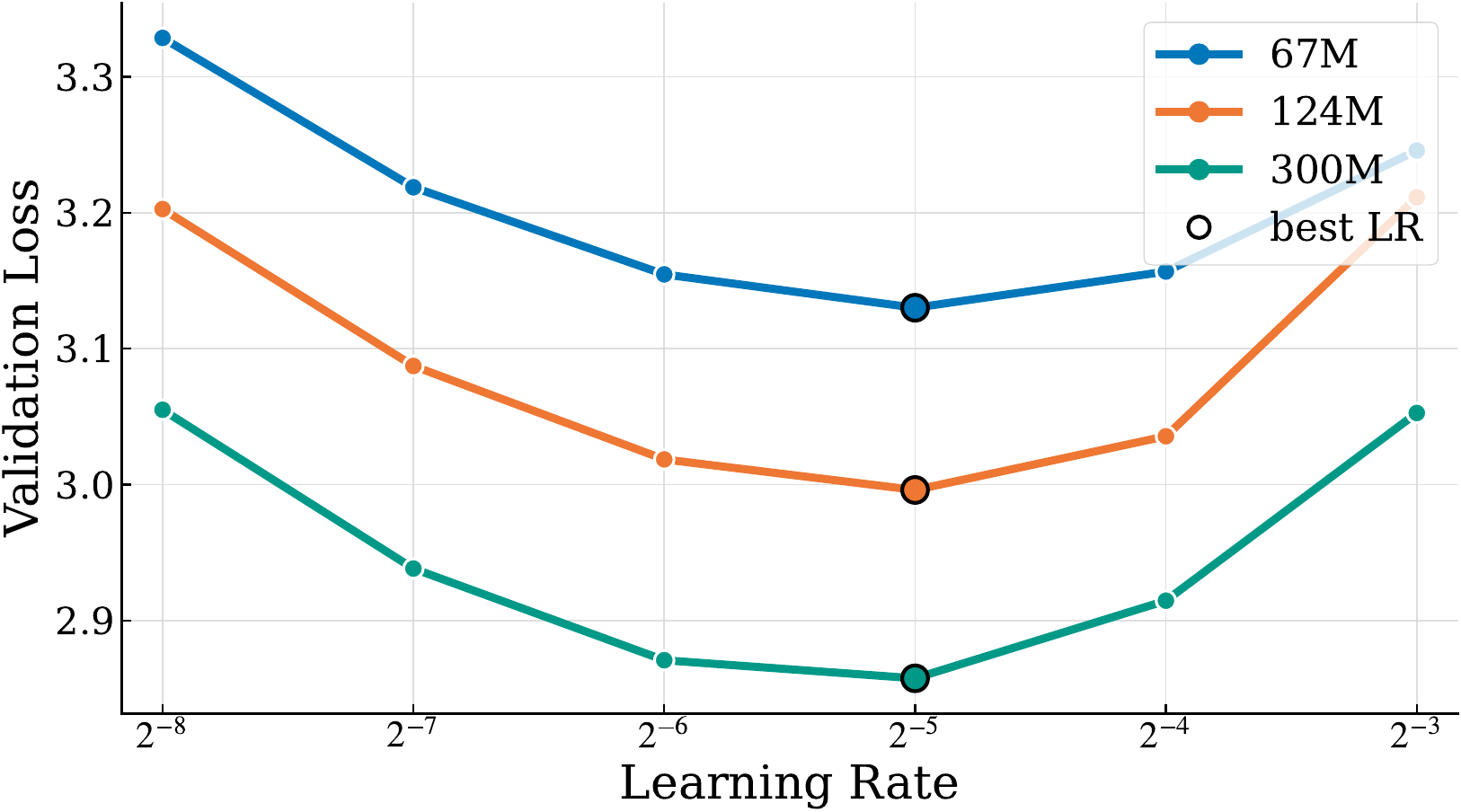}
		\end{minipage}
		\caption{Width transfer.}
		\label{fig:lr_transfer}
	\end{subfigure}
	\caption{
		Validation loss across learning rates.
		\textbf{(a)} Qwen2-0.5B trained by \ourAlg{} with polynomial angular decay ($\kappa_t \propto \nicefrac 1 t$), without a decay ($\kappa_t \equiv 1$), and \textsc{NorMuon}.
		\textbf{(b)} Width transfer for \ourAlg{} at model widths $512$, $768$, and $1280$.
	}
	\label{fig:combined_qwen_lr_transfer}
    \vspace{-3mm}
\end{figure}

\paragraph{PlainLM.} We adopt the architecture from \citet{ajroldi2024plainlm}, which is based on a nanoGPT \citep{karpathy2022nanogpt} implementation modified to include recent architectural improvements such as RoPE \citep{su_rope}, RMSNorm normalization \citep{zhang_rmsnorm}, and SwiGLU activations \citep{shazeer_glu_2020}. To compare performance, we sweep the learning rate of \textsc{AdamW}, \textsc{NorMuon}, and \muown\ across \texttt{[5e-4, 1e-3, 2e-3, 4e-3, 8e-3]}. \ourAlg\ was swept across \texttt{[5e-3, 1e-2, 2e-2, 4e-2, 8e-2]}. Other parameters are set to their usual values and can be found in Table~\ref{tab:optimizer_hyperparameters}. We start decay branches after $2.5$B, $5$B, $7.5$B, and $10$B tokens and the results can be found in Figure~\ref{fig:teaser}. \ourAlg\ outperforms all baselines for each horizon and achieves a speedup of approximately $2\times$ over \textsc{AdamW}, and $1.5\times$ over \textsc{NorMuon} and \muown.

Next we empirically examine the \emph{angular displacement} discussed in Section~\ref{subsec:angular}. Therefore we train the 124M model for Chinchilla-optimal 2.5B tokens \citep{hoffmann_training_2022} with \textsc{Muon} with and without weight decay, \textsc{NorMuon}, \muown, as well as \ourAlg\ with ($\kappa_t \propto \nicefrac 1 t$) and without ($\kappa_t \equiv 1$) angular decay. We perform a learning-rate sweep over \texttt{[1e-3, 2e-3, 4e-3, 6e-3, 8e-3, 1e-2]} and report the optimal run for each optimizer. Figure~\ref{fig:124M-angular} makes \ourAlg's angular mechanism visible: with a constant step without angular decay the per-step rotation stays roughly flat, whereas every decaying-angle method---including \ourAlg{}---anneals $\overline{\theta}$ smoothly towards zero. \Cref{fig:124M-angular} also corroborates the three regimes mentioned in \Cref{tab:regimes}. First, adding weight decay to \textsc{Muon} ($\lambda=0.1$) lifts its angular displacement to a sustained, non-vanishing plateau instead of letting it decay. Second, \muown{} results in an annealed angular update, while \ourAlg\ allows for arbitrary annealing. At a well-tuned learning rate this annealing lets \ourAlg{} take large, aggressive directional steps and reach the lowest perplexity of all methods (\Cref{fig:124M-lr}).

Finally we study the \emph{learning rate transfer} across widths. Following \citet{pethick_training_2025}, we train our 124M base architecture for 5100 iterations with model widths 512, 768, and 1280 (i.e., total parameter count of 67M, 124M, and 300M). For each width, we run \ourAlg\ with learning rates $\eta = 2^{-8}, \dots, 2^{-3}$ and report the final validation loss in Figure~\ref{fig:lr_transfer}. We note that the optimal learning rate perfectly transfers across widths, validating the effectiveness of the scale factor discussed in Section~\ref{sec:angular_muown.angular_muown.scale_factor}.

\paragraph{Qwen2-0.5B.} The same pattern transfers to a Qwen2-0.5B architecture (\Cref{tab:qwen_results}). At its well-tuned learning rate, the polynomial schedule attains the best loss overall and improves over \textsc{NorMuon} for every $\eta \ge 2\times 10^{-2}$, whereas the unscheduled variant degrades sharply as $\eta$ grows. Making the angular step size explicit thus shifts the optimum towards larger, more aggressive steps, consistent with the angular-annealing view of \Cref{sec:motivation_and_prelims}.

\begin{wraptable}{r}{0.45\textwidth}
	\vspace{-0.75\baselineskip}
	\centering
	\setlength{\tabcolsep}{3.5pt}
	\renewcommand{\arraystretch}{0.95}
	\begin{tabular}{@{}lccc@{}}
		\toprule
		\multirow{2}{*}{$\eta$}
		& \multicolumn{2}{c}{\ourAlg{}}
		& \multirow{2}{*}{\textsc{Muon}} \\
		\cmidrule(lr){2-3}
		& $\sqrt{\max(m,n)}$  & $ \sqrt{\max(1, m/n)}$ & \\
		\midrule
		  \texttt{3e-4} & $2.6124$          & ---               & $2.5297$          \\
		\texttt{6e-4} & $2.5278$          & ---               & $\mathbf{2.5255}$ \\
		\texttt{1e-3}      & $2.4887$          & $2.6952$          & $2.5263$          \\
		\texttt{2e-3} & $\mathbf{2.4807}$ & $2.5901$          & $2.5353$          \\
		\texttt{4e-3} & ---               & $2.5198$          & ---               \\
		\texttt{8e-3} & ---               & $\mathbf{2.4914}$ & ---               \\
		\bottomrule
	\end{tabular}
	\caption{Learning rate sweep on a DeepSeek-V3-style MoE-model \citep{liu_muon_2025} with 1.1B-A370M parameters. For \ourAlg{}, we ablate the spectral condition scaling $s_{m, n} = \sqrt{\max(1, m/n)}$ and the RMS-matching scaling $s_{m,n} = \sqrt{\max(m, n)}$. }
	\label{tab:angular-muown-moe}
	\vspace{-\baselineskip}
\end{wraptable}
\paragraph{Mixture-of-Experts.} Moving beyond dense transformer models, we extend our experimental study to Mixture-of-Experts (MoE) models \citep{shazeer_outrageously_2017} using the widely adopted Megatron-LM codebase \citep{megatron-lm}. We consider a 1.1B DeepSeek-V3-style MoE \citep{deepseek-ai_deepseek-v3_2025} with 370M active parameters (1.1B-A370M), following the architectural choices of the Moonlight recipe for Muon \citep{liu_muon_2025}. 
The improvement of \ourAlg\ over Muon persists in this sparse model setup. For \ourAlg, some of the best points lie at the edge of the tested grid, so the reported MoE losses may be conservative. We kept the learning-rate grids comparable across algorithms to avoid giving either method a larger tuning budget.

\section{Related Work}
\label{sec:related_work}
\paragraph{Matrix-aware optimizers.}
Early methods exploiting the matrix structure of feedforward layers include \textsc{Shampoo} \citep{gupta_shampoo_2018} and \textsc{Soap} \citep{vyas_soap_2024}. Another notable work is \textsc{Muon} \citep{jordan2024muon}, which performs spectral-norm steepest descent via Newton-Schulz orthogonalization. Descendants of \textsc{Muon} include \textsc{NorMuon} \citep{li_normuon_2025}, which applies neuron-wise adaptive scaling of the update after orthogonalization, and \textsc{Muown} \citep{lion_muown_2026} which splits each weight into row magnitudes and a direction component internally. While \textsc{Muown} updates the direction in the ambient space, leaving the row-scales to drift freely, our parameterization fixes them to unit norm.

\paragraph{Weight normalization and reparameterization.} \citet{salimans_weight_2016} introduce the weight-norm parameterization to decouple magnitude learning from the direction, making the direction component of the weights scale-invariant in the sense that any positive rescaling of the direction component does not alter the underlying function. 
\citet{NormMatters2018Hoffer} present a modification of the parameterization to fix the learnable magnitude to a constant value.

\paragraph{Scale invariance, angular step size, and weight decay.}
For scale-invariant neural networks, the effective step size is the angular rate of change, which is governed implicitly through the weight norm \citep{laarhoven_l2_2017, NormMatters2018Hoffer}. \citet{kosson_rotational_2024} formalize this as a rotational equilibrium in which weight decay balances norm growth so the angular update reaches a stable value. They propose a rotational wrapper whose main goal is to provide an update whose average angular update matches the optimizer's predicted equilibrium rotation. In their analysis, this equilibrium rotation is induced by the choice of the learning rate in conjunction with the weight decay strength. These works focus on the \emph{implicitly} induced angular step size behavior. While \citet{laarhoven_l2_2017} focuses on the implicit weight norm dynamics, \citet{kosson_rotational_2024} attempt to mimic the induced rotational equilibrium throughout training by enforcing the predicted average angular rotation directly. In contrast, \ourAlg{} makes the angular step size an explicit and \emph{independently} scheduled quantity, decoupled from the radial magnitude update and its schedule, which renders weight decay redundant on the directions.

\paragraph{Concurrent work.} 
Very recently we became aware of the concurrent blog post \citep{hagele2026improving}. Similar to this work, \citet{hagele2026improving} keep the directional factor of a magnitude-direction parameterization on a fixed norm. The main overlap with \ourAlg\ is therefore not the magnitude-direction decomposition itself, which is already used by \textsc{Muown} \citep{lion_muown_2026}, but the fixed-norm treatment of the direction and the resulting direct control of angular updates.
The methods differ in their geometry: \ourAlg\ uses the row-wise Riemannian gradient $\GradOp_{U} L = \operatorname{Diag}(g) \bigl(G - \operatorname{Diag}(\operatorname{diag}(G U^\top)) U\bigr)$, and introduces an explicit angular multiplier $\kappa_t$. 
By contrast, \citet{hagele2026improving} use the ambient chain-rule gradient for the on-sphere direction, $G_{\widehat W} = \operatorname{Diag}(\gamma_{\rm row})\, G\, \operatorname{Diag}(\gamma_{\rm col})$ for \(W=\operatorname{Diag}(\gamma_{\rm row})\widehat W \operatorname{Diag}(\gamma_{\rm col})\), and project the direction $\widehat W$ back to the chosen sphere after its update. Their reported experiments mainly use endpoint-dependent linear decay of the learning rates, whereas \ourAlg\ is designed as a drop-in replacement for the standard WSD schedule, with the angular schedule controlled separately through $\kappa_t$.

\section{Conclusion and Limitations}
\label{sec:conclusion_and_limitations}
We observe that \muown’s stored row norm implicitly anneals the angular step size. Based on this observation we propose \ourAlg, which makes this mechanism explicit by optimizing directly on the row-oblique manifold and scheduling the angular multiplier explicitly. Across the studied language-model pre-training settings, this explicit control yields considerable improvements over \textsc{AdamW}, \textsc{NorMuon}, and \muown.

The main limitation is that our evaluation is centered on language-model pre-training and limited to medium model sizes. A broader study of architectures, training regimes, and scaling to 100B and larger parameter models remains open.

\subsubsection*{Acknowledgments}
This work was supported under project ID a0184 as part of the Swiss AI Initiative, through a small grant from the ETH Domain and computational resources provided by the Swiss National Supercomputing Centre (CSCS) under the Alps infrastructure. Kai Lion is supported by Swiss National Science Foundation (SNSF) Sinergia Funding No. 216600. Florian Hübler acknowledges financial support from the ETH research grant and Swiss National Science Foundation (SNSF) Project Funding No. 200021-207343. Antonio Orvieto acknowledges the financial support of the Hector Foundation. Niao He is supported by an ETH research grant funded through the ETH Zurich Foundation and by an SNSF Starting Grant.

\printbibliography
\clearpage
\appendix
\section{Experimental Details}
\label{appdx:experimental-details}

\paragraph{Common setup.} We adopt the experimental setup of \citet{ajroldi2024plainlm}, which is based on a nanoGPT \citep{karpathy2022nanogpt} implementation augmented with recent architectural improvements such as RoPE \citep{su_rope}, RMSNorm normalization \citep{zhang_rmsnorm}, and SwiGLU activations \citep{shazeer_glu_2020}. All models are pre-trained on FineWeb-Edu \citep{lozhkov2024fineweb-edu} with a warmup-stable-decay learning-rate schedule \citep{hu_minicpm_2024}, and we report perplexity (equivalently validation loss) at a fixed token budget. Following \textsc{Muown}, the matrix-aware optimizer acts on all 2D hidden weight matrices (every 2D weight except the token embeddings and the LM head); the remaining parameters are handled by AdamW. For \ourAlg{} the row scales $g$ are updated with Adam under the base learning rate $\eta_t$, while the directional update orthogonalizes the projected momentum with Newton--Schulz and rescales it by the angular learning-rate multiplier $\kappa_t$ (the radial and angular dynamics are therefore scheduled independently, cf.\ \Cref{sec:riemannian_muown}).

\paragraph{124M checkpoint-branching (Figure~\ref{fig:teaser}).} The hyperparameters used by the different optimizers can be found in Table~\ref{tab:optimizer_hyperparameters}.

\newcommand{\NA}{\textemdash}

\begin{table}
    \centering
    \caption{Optimizer hyper-parameters, underlines values denote the tuned learning rates.}
    \label{tab:optimizer_hyperparameters}
    \small
    \setlength{\tabcolsep}{5pt}
    \renewcommand{\arraystretch}{1.15}

    \begin{tabularx}{\linewidth}{lcccccc}
        \toprule
        Optimizer
        & Learning rate $\eta$
        & $\beta_1$
        & $\beta_2$
        & $\lambda$
        & \makecell{Orth.\\steps}
        & \texttt{decay\_scale}\\
        \midrule

        \textsc{AdamW}
        & \texttt{5e-4, 1e-3, \underline{2e-3}, 4e-3, 8e-3}
        & 0.9
        & 0.95
        & 0.1
        & \NA
        & \NA \\

        \textsc{NorMuon}
        & \texttt{5e-4, 1e-3, 2e-3, \underline{4e-3}, 8e-3}
        & 0.95
        & 0.95
        & \NA
        & 5
        & \NA \\

        \muown
        & \texttt{5e-4, 1e-3, 2e-3, \underline{4e-3}, 8e-3}
        & 0.95
        & \NA
        & \NA
        & 5
        & \NA \\

        \ourAlg{}
        & \texttt{5e-3, 1e-2, 2e-2, \underline{4e-2}, 8e-2}
        & 0.95
        & \NA
        & \NA
        & 5
        & $0.001$ \\

        \bottomrule
    \end{tabularx}
\end{table}

\paragraph{124M learning-rate sweep (\Cref{fig:124M}).} We train a $124$M-parameter model for $2.5$B tokens and sweep a log-spaced learning-rate grid (\Cref{fig:124M-lr}), comparing \textsc{Muon}, NorMuon \citep{li_normuon_2025}, \textsc{Muown}, and \ourAlg{} both with and without its angular schedule. \textsc{Muon} uses weight decay $\lambda = 0.1$ and all other optimizers use $\lambda = 0$, the settings found optimal in the \textsc{Muown} study. \ourAlg{}'s angular schedule is the inverse-polynomial $\kappa_t^{\mathrm{poly}} = (1 + c\,(t - t_w)_+)^{-p}$, using the default values $c = 10^{-3}$ and $p = 1$. The angular-displacement curves in \Cref{fig:124M-angular} each correspond to the best-tuned learning rate per optimizer (lowest final validation loss). The per-step row rotation $\theta = \arccos\langle \mathbf{u}, \mathbf{u}'\rangle$ is read directly from the stored weight rows before and after each step.

\paragraph{Qwen2-0.5B sweep (Table~\ref{tab:qwen_results}).} To probe transfer beyond the architecture above, we additionally train a Qwen2-0.5B \citep{yang_qwen2_2024} model without weight decay ($\lambda = 0$). We sweep $\eta \in \{4, 6, 8\} \times 10^{-3} \cup \{1, 2, 4\} \times 10^{-2}$ and compare \ourAlg{} with its polynomial angular schedule against the constant-angular variant ($\kappa_t \equiv 1$) and NorMuon.

\paragraph{MoE Experiments (Table~\ref{tab:angular-muown-moe}).} The model under consideration is a downscaled variant of Moonshot's Moonlight-16B-A3B \citep{liu_muon_2025} reduced to 14 transformer layers (1 dense + 13 MoE), hidden size 1024, and 32 experts (top-6, sigmoid router with learned expert bias) plus one shared expert, for 1.1B total and 370M active parameters. We keep Moonlight's DeepSeek-V3-style MoE design but replace MLA with standard grouped-query attention (16 heads, 4 KV groups) to avoid latent-projection special cases in the optimizer comparison. Training runs in patched Megatron-LM \citep{megatron-lm} on FineWeb-Edu (sample-100BT, GPT-NeoX tokenizer, sequence length 4096) at a global batch of 256 ($\approx$1.05M tokens/step) for 30B tokens in total. For the optimizer settings, we mostly follow \cite{liu_muon_2025}: \textsc{Muon} and \ourAlg{} are set with momentum 0.95 with Nesterov, shape scaling of $\sqrt{\max(m, n)}$, 5 Newton–Schulz steps, and weight decay 0.1 for Muon. Moreover, routers, embeddings, norms, and biases are handled by Adam in all cases. For \ourAlg{}, we additionally consider the shape scaling factor of $s_{m,n} = \sqrt{\max(1, m/n)}$, which results in a dimension-independent rotation.

\section{Proof of Theorem \texorpdfstring{\ref{thm:conv}}{3.2}}
\label{sec:app.missing_proofs}

\begin{figure}[t]
	\vspace{-1.5\baselineskip}
	\begin{algorithm}[H]
		\caption{Idealized Directional \ourAlg\ Update}
		\label{alg:idealised_rie_muown}
		{\linespread{1.3}\selectfont
		\small
		\begin{algorithmic}
			\Require $\iterationU \fin \in \Mc$, $\iterationM 0 \gets \GradOp \mathcal L (\iterationU 1, \iterationxi 1)$, stepsize $\eta$, momentum $\beta$ \vspace{2mm}
            \State $\Mt \gets \beta \Mtm + (1-\beta)\mathrm{grad}\mathcal L(\Ut, \xit)$
            \State $\Ot \gets \argmin_{\Dbf \in \mathcal T_{\Ut} \mathcal M, \opnorm{\Dbf}\leq 1} \inner{\Mt}{\Dbf}$ \Comment{Tangent-space aware orthogonalization}
            \State $\Utp \gets \mathrm{Retr}_{\Ut}\pare{\ssi \Ot}$ \Comment{Retraction}
		\end{algorithmic}
		}
	\end{algorithm}
\end{figure}

In this section we provide the missing proof for Section~\ref{sec:angular_muown.conv_guarantee}. The proof follows the classical arguments from \citep{MomentumImprovesNormalized2020cutkosky}.

\paragraph{Notation.} Let $m, n \in \Ngeq$ and for matrices $\Abf \in \R^{m \times n}$ let $\Abf_i\in\R^n$ denote the i-th row of $\Abf$. Let $\Mc = \set{\Abf \in \R^{m \times n} \suchthat \rownorm{\Abf}=1}$ denote the row-oblique manifold,
\begin{equation*}
    \Tan \Ubf \coloneqq \set{\Abf \in \R^{m \times n} \suchthat \langle \Abf_i, \Ubf_i \rangle = 0 \text{ for all } i \in[m]}
\end{equation*}
the tangent space at $\Ubf \in \Mc$, and $\Proj \Ubf \Abf = \Abf - \Diag{\diag{\Abf \Ubf^\top}} \Ubf$ the orthogonal projector onto $\Tan \Ubf$. Furthermore we denote the tangent-space restricted dual norm as $$\norm{\Abf}_{\Ubf, *} \coloneqq \max_{\Vbf \in \Tan \Ubf, \opnorm{\Vbf} \leq 1} \langle \Abf, \Vbf \rangle,$$
and, for $\Ubf \in \Mc, \Abf \in \Tan \Ubf$, use the retraction $\Retr \Ubf \Abf = \Pi_\Mc \pare{\Ubf + \Abf}$.

Now we first introduce the required assumptions.

\begin{assum}\label{assum:lb}
    The objective function is lower bounded with $\mathcal L\pare{\iterationU \fin} - \inf_{\Ubf \in \Mc} \mathcal L(\Ubf) \leq \Dz$.
\end{assum}

\begin{assum}\label{assum:smooth}
    For all $\Ubf \in \Mc$ and $\Abf \in \Tan \Ubf$ with $\Vbf\coloneqq \Retr \Ubf \Abf$ we have
    \begin{align*}
        &\mathcal L \pare{\Vbf} \leq \mathcal L (\Ubf) + \langle \GradOp \mathcal L (\Ubf), \Abf \rangle + \frac{L_{\mathcal L}}{2}\opnorm{\Abf}^2,\\
        &\nucnorm{\GradOp \mathcal L(\Ubf) - \GradOp \mathcal L \pare{\Vbf}} \leq L_{g} \opnorm \Abf,
    \end{align*}
    and denote $L \coloneqq \max\set{L_{\mathcal L}, L_g}$.
\end{assum}
\begin{assum}\label{assum:noise}
    The gradient oracle is unbiased and has finite variance, i.e., for all $\Ubf \in \Mc$,
    \begin{equation*}
        \Exp{\GradOp \mathcal L \pare{\Ubf, \xi}} = \GradOp \mathcal L \pare{\Ubf} \qquad \text{and} \qquad \Exp{\frobnorm{\GradOp \mathcal L \pare{\Ubf, \xi} - \GradOp \mathcal L \pare{\Ubf}}^2} \leq \sigma^2.
    \end{equation*}
\end{assum}

Additionally we require the following standard inequalities.

\begin{lemma}\label{lem:technical}
    For all $\Ubf \in \Mc$ and $\Abf \in \Tan \Ubf$ we have
    \begin{equation*}
        \frobnorm{\Abf} \leq \norm{\Abf}_{\Ubf, *} \leq \nucnorm{\Abf} \leq \sqrt{\min\set{m, n}} \frobnorm{\Abf}.
    \end{equation*}
\end{lemma}

Now we are ready to provide the descent lemma.

\begin{lemma}[Descent Lemma]\label{lem:descent_lemma}
    Let Assumptions \ref{assum:lb} and \ref{assum:smooth} hold. Then the iterates generated by Algorithm~\ref{alg:idealised_rie_muown} satisfy
    \begin{equation*}
        \iSum \dnorm{\GradOp \mathcal L(\Ut)}
        \leq \frac {\Dz} \ssi + \frac{\ssi L T}{2} + 2 \iSum \dnorm{\Mt - \GradOp \mathcal L(\Ut)}
    \end{equation*}
\end{lemma}
\begin{proof}
    The proof arguments follow \citep{MomentumImprovesNormalized2020cutkosky}, we include it for completeness. By Assumption~\ref{assum:smooth} we have
    \begin{equation*}
        \mathcal L(\Utp) - \mathcal L(\Ut) 
        \leq \ssi \langle \GradOp \mathcal L(\Ut), \Ot \rangle + \frac {\ssi^2 L_{\mathcal L}} 2 \opnorm{\Obf}^2 
        \leq \ssi \langle \Mt, \Ot \rangle - \ssi \langle \Ebf_t , \Ot \rangle+ \frac{\ssi^2 L}{2},
    \end{equation*}
    where $\Ebf_t \coloneqq \Mt - \GradOp \mathcal L \pare{\Ut}$. 
    The Hölder-inequality and $\dnorm{\Mt} \geq \dnorm{\GradOp \mathcal L (\Ut)} - \dnorm{\Ebf_t}$ further yield
    \begin{equation*}
        \langle \Mt, \Ot \rangle - \langle \Ebf_t , \Ot \rangle
        \leq - \dnorm{\Mt} + \dnorm{\Ebf_t}
        \leq - \dnorm{\GradOp \mathcal L (\Ut)} + 2 \dnorm{\Ebf_t}.
    \end{equation*}
    Summing up and using Assumption~\ref{assum:lb} yields the claim.
\end{proof}

Next we derive the noise bound, which follows the same arguments as \muown's \citep{lion_muown_2026}. We include it for completeness.

\begin{lemma}[Deviation Bound]\label{lem:deviation_bound}
    Let Assumptions \ref{assum:smooth} and \ref{assum:noise} hold, and denote $r \coloneqq \min\set{m, n}$. Then the iterates generated by Algorithm~\ref{alg:idealised_rie_muown} satisfy
    \begin{equation*}
        \iSum \Exp{\nucnorm{\Et}}
        \leq \frac{\sqrt r\sigma}{1-\beta} + \sqrt r \sigma T \sqrt{1-\beta} + \frac{\ssi L T}{1-\beta}.
    \end{equation*}
\end{lemma}
\begin{proof}
    For notational conciseness, define
    \begin{equation*}
        \Gt \coloneqq \GradOp \mathcal L\pare{\Ut, \xit}, \quad \Zt \coloneqq \Gt - \GradOp \mathcal L(\Ut), \qquad \St \coloneqq \GradOp \mathcal L\pare{\Utm} - \GradOp \mathcal L(\Ut).
    \end{equation*}
    Then we can unroll $\Mt \coloneqq \beta \Mtm + (1-\beta) \Gt$ into
    \begin{equation*}
        \Et = 
            \beta^{\ii - 1} \iterationE \fin 
            + (1-\beta) \sum_{\tau = 2}^{\ii} \beta^{\ii - \tau} \pare{\iterationG \tau - \GradOp \mathcal L(\iterationU \tau)}
            + \sum_{\tau = 2}^\ii \beta^{\ii - \tau + 1} \iterationS \tau.
    \end{equation*}
    Next, Assumption~\ref{assum:smooth} yields $\nucnorm{\St} \leq \ssi L_g \opnorm{\Otm} \leq \ssi L$ and thus
    \begin{equation*}
        \nucnorm{\sum_{\tau = 2}^\ii \beta^{\ii - \tau + 1} \iterationS \tau}
        \leq \sum_{\tau = 2}^\ii \beta^{\ii - \tau + 1} \nucnorm{\iterationS \tau}
        \leq \frac{\ssi L}{1-\beta}.
    \end{equation*}
    Next note that $\iterationG \tau - \GradOp \mathcal L(\iterationU \tau)$ is a martingale difference sequence and thus \citep{MomentumImprovesNormalized2020cutkosky} implies
    \begin{align*}
        \Exp{\nucnorm{\sum_{\tau = 2}^{\ii} (1-\beta)\beta^{\ii - \tau} \pare{\iterationG \tau - \GradOp \mathcal L(\iterationU \tau)}}}
        &\leq \sqrt{r} \Exp{\frobnorm{\sum_{\tau = 2}^{\ii} (1-\beta)\beta^{\ii - \tau} \pare{\iterationG \tau - \GradOp \mathcal L(\iterationU \tau)}}}\\
        &\leq \sqrt{r} (1-\beta) \sqrt{\sum_{\tau = 1}^{\ii} \beta^{2(\ii - \tau)} \sigma^2}\\
        &\leq \sqrt{r} \sigma \sqrt{1-\beta}.
    \end{align*}
    Combining the above with our initialization $\iterationM \fin = \iterationG \fin$ yields
    \begin{equation*}
        \iSum \Exp{\nucnorm{\Et}}
        \leq \frac{\sqrt r\sigma}{1-\beta} + \sqrt r \sigma T \sqrt{1-\beta} + \frac{\ssi L T}{1-\beta}
    \end{equation*}
    and thus the claim.
\end{proof}

Finally we provide the proof for Theorem~\ref{thm:conv}.

\begin{proof}[Proof of Theorem~\ref{thm:conv}]
    The proof follows the same steps as \muown's proof \citep[Appendix B.2]{lion_muown_2026}. By Lemma~\ref{lem:descent_lemma} we have
    \begin{align*}
        \iSum \Exp{\dnorm{\GradOp \mathcal L(\Ut)}}
        &\leq \frac {\Dz} \ssi + \frac{\ssi L T}{2} + 2 \iSum \Exp{\dnorm{\Mt - \GradOp \mathcal L(\Ut)}}\\
        &\leq \frac {\Dz} \ssi + 3 \frac{\ssi L T}{1-\beta} + 2 \sigma \sqrt r \pare{\frac 1 {1-\beta} + T \sqrt{1-\beta}},
    \end{align*}
    where we used Lemma~\ref{lem:deviation_bound} in the second step. Dividing by $T$ and our choice of stepsize implies
    \begin{equation*}
        \frac 1 T \iSum \Exp{\dnorm{\GradOp \mathcal L(\Ut)}}
        \leq 4 \sqrt{\frac{\Dz L}{T (1-\beta)}} + \frac{2 \sigma \sqrt r }{T(1-\beta)} + 2 \sigma \sqrt r \sqrt{1-\beta}.
    \end{equation*}
    Our choice of momentum guarantees
    \begin{align*}
        \sqrt{\frac{\Dz L}{T (1-\beta)}} 
        &\leq \sqrt{\frac{\Dz L}{T}} + \pare{\frac{r \Dz L \sigma^2}{T}}^{\nicefrac 1 4}\\
        \sigma \sqrt r \sqrt{1-\beta} 
        &\leq \frac{\sqrt r \sigma}{T^{\nicefrac 1 3}} + \pare{\frac{r \Dz L \sigma^2}{T}}^{\nicefrac 1 4}\\
        \frac{\sigma \sqrt r}{T (1-\beta)}
        &\leq \frac{\sqrt r \sigma}{T^{\nicefrac 1 3}},
    \end{align*}
    and thus
    \begin{equation*}
        \frac 1 T \iSum \Exp{\dnorm{\GradOp \mathcal L(\Ut)}}
        \leq 4 \sqrt{\frac{\Dz L}{T}} + 7 \pare{\frac{r \Dz L \sigma^2}{T}}^{\nicefrac 1 4} + 4 \frac{\sqrt r \sigma}{T^{\nicefrac 1 3}}.
    \end{equation*}
    This proves a stronger statement, but for ease of presentation we use Lemma~\ref{lem:technical} to convert $\dnorm \cdot$ to the Frobenius norm, which finishes the proof.
\end{proof}

\end{document}